\documentclass{article}

    \usepackage[nonatbib,preprint]{neurips_2025}

\usepackage[utf8]{inputenc} 
\usepackage[T1]{fontenc}    
\usepackage{url}            
\usepackage{booktabs}       
\usepackage{amsfonts}       
\usepackage{nicefrac}       
\usepackage{xcolor}         

\usepackage{algorithm}
\usepackage{algorithmic}
\usepackage{csquotes}
\usepackage{amsmath}
\usepackage{centernot}  
\usepackage{enumitem}
\usepackage{tabularx}
\usepackage{adjustbox}
\usepackage{multirow}
\usepackage{pgfplots}
\usepgfplotslibrary{groupplots}
\pgfplotsset{compat=1.18}
\usepackage{listings}
\usepackage{natbib}

\usepackage{hyperref}

\usepackage{subcaption}

\usepackage{amssymb}
\usepackage{mathtools}
\usepackage{amsthm}
\theoremstyle{plain}
\newtheorem{theorem}{Theorem}[section]

\theoremstyle{definition}

\theoremstyle{remark}
\newtheorem{remark}[theorem]{Remark}

\title{Rethinking GSPO: The Perplexity-Entropy Equivalence}

\author{%
  \textbf{Chi Liu} \\
  \texttt{chiliu@whu.edu.cn}
}

\begin{document}
\maketitle

\begin{abstract}

We provide a new perspective on GSPO's length-normalized importance ratios by establishing their connection to information-theoretic quantities. We show that GSPO's sequence-level weight $s(\theta) = (\pi_\theta/\pi_{\theta_{\text{old}}})^{1/|y|}$ can be equivalently expressed as the inverse perplexity ratio $\text{PPL}_{\theta_{\text{old}}}/\text{PPL}_\theta$ and as the exponential cross-entropy change $\exp(\Delta H)$. While the perplexity-entropy relationship follows from standard definitions, this observation provides a useful lens for understanding GSPO: the algorithm weights policy gradient updates by perplexity ratios, offering an information-theoretic interpretation of the importance weights. This perspective helps explain GSPO's empirical properties, including log-domain variance reduction through geometric averaging and stability in training mixture-of-experts models. We validate the mathematical equivalences and variance predictions through controlled experiments on mathematical reasoning tasks.


\end{abstract}
\section{Introduction}

Reinforcement learning has become pivotal for scaling language models, enabling capabilities from competition-level mathematics to complex reasoning chains \citep{o1,anthropic2024,yang2025qwen3technicalreport,ouyang2022traininglanguagemodelsfollow}. Policy gradient methods like PPO \citep{schulman2017proximalpolicyoptimizationalgorithms} and their variants \citep{rafailov2024directpreferenceoptimizationlanguage,lee2024rlaifvsrlhfscaling} have achieved remarkable success, yet understanding their theoretical properties remains an active research area. Among recent algorithms, GSPO \citep{zheng2025groupsequencepolicyoptimization} and GRPO \citep{shao2024deepseekmathpushinglimitsmathematical} stand out for their stability through sequence-level optimization with length-normalized importance ratios, yet the theoretical justification for this design choice remains unclear.

GSPO's defining feature—taking the $|y|$-th root of likelihood ratios—was justified as variance reduction:
\begin{equation}
s(\theta) = \left(\frac{\pi_\theta(y|x)}{\pi_{\theta_{\text{old}}}(y|x)}\right)^{1/|y|}
\end{equation}
But why should the geometric mean be the principled choice? Is this fortunate engineering or a deeper principle?

We provide a new perspective on this design by observing that length normalization connects importance ratios to perplexity. Specifically:
\begin{equation}
s(\theta) = \frac{\text{PPL}_{\theta_{\text{old}}}(y|x)}{\text{PPL}_{\theta}(y|x)} = \exp(\Delta H)
\end{equation}
where $\Delta H$ is the cross-entropy change. While the perplexity-entropy relationship $\text{PPL} = \exp(H)$ follows from standard definitions, recognizing that GSPO's importance weights equal perplexity ratios provides a useful interpretive lens: the algorithm weights policy gradients by how much the new policy improves perplexity over the old policy.

Our main contributions are:
\begin{itemize}[topsep=0pt,itemsep=1pt,parsep=0pt,leftmargin=*]
\item \textbf{Perplexity-ratio observation:} We show that GSPO's length-normalized importance ratio $s(\theta)$ equals the inverse perplexity ratio, connecting sequence-level RL to language modeling metrics.
\item \textbf{Variance analysis:} We establish $O(1/L)$ variance reduction in log-domain through the geometric averaging inherent in the perplexity formulation.
\item \textbf{Empirical insights:} We explain observed phenomena such as MoE training stability (geometric averaging dampens routing fluctuations) and long-sequence advantages (variance reduction scales with length).
\item \textbf{Experimental validation:} We verify the mathematical equivalences and variance predictions through controlled experiments, finding 3.1× deviation from idealized theory that we attribute to token correlations and length heterogeneity.
\end{itemize}

This information-theoretic perspective complements existing understanding of GSPO, offering new intuitions for practitioners and suggesting connections between policy optimization and perplexity minimization.
\section{Preliminaries}
\vspace{-2mm}
\subsection{Notation}

We denote an autoregressive language model parameterized by $\theta$ as a policy $\pi_\theta$. Given a query $x$ and response $y = (y_1, \ldots, y_{|y|})$:
\begin{align}
\pi_\theta(y|x) &= \prod_{t=1}^{|y|} \pi_\theta(y_t|x, y_{<t})\\
\text{PPL}_\theta(y|x) &= [\pi_\theta(y|x)]^{-1/|y|} = \exp(H_\theta(y|x))\footnote{The perplexity-entropy relationship follows from: $\text{PPL}_\theta(y|x) = \exp\left(-\frac{1}{|y|} \log \pi_\theta(y|x)\right) = \exp(H_\theta(y|x))$ where $H_\theta(y|x) = -\frac{1}{|y|}\sum_{t=1}^{|y|} \log \pi_\theta(y_t|y_{<t}, x)$ is the cross-entropy \citep{cover2006elements,perplexity_language_models}.}\\
H_\theta(y|x) &= -\frac{1}{|y|}\log \pi_\theta(y|x)
\end{align}
establishing the standard connection $\text{PPL}_\theta = \exp(H_\theta)$ \citep{shannon1948mathematical,cover2006elements}.

\subsection{GRPO and GSPO}

Both GRPO \citep{shao2024deepseekmathpushinglimitsmathematical} and GSPO \citep{zheng2025groupsequencepolicyoptimization} use group-relative advantages: for $G$ responses $\{y_i\}_{i=1}^G$,
\begin{equation}
\hat{A}_i = \frac{r(x, y_i) - \text{mean}(\{r(x, y_j)\}_{j=1}^G)}{\text{std}(\{r(x, y_j)\}_{j=1}^G)}
\end{equation}

GRPO applies token-level importance ratios $w_{i,t}(\theta) = \pi_\theta(y_{i,t}|x, y_{i,<t})/\pi_{\theta_{\text{old}}}(y_{i,t}|x, y_{i,<t})$:
\begin{equation}
\mathcal{J}_{\text{GRPO}} = \mathbb{E}\left[\frac{1}{G}\sum_{i=1}^G \frac{1}{|y_i|}\sum_{t=1}^{|y_i|} \min\left(w_{i,t}(\theta)\hat{A}_i, \text{clip}(w_{i,t}, 1-\varepsilon, 1+\varepsilon)\hat{A}_i\right)\right]
\end{equation}
GSPO uses length-normalized sequence-level ratios:
\begin{equation}
\mathcal{J}_{\text{GSPO}} = \mathbb{E}\left[\frac{1}{G}\sum_{i=1}^G \min\left(s_i(\theta)\hat{A}_i, \text{clip}(s_i, 1-\varepsilon, 1+\varepsilon)\hat{A}_i\right)\right]
\end{equation}
where $s_i(\theta) = (\pi_\theta(y_i|x)/\pi_{\theta_{\text{old}}}(y_i|x))^{1/|y_i|}$.

The gradient structures reveal the key difference. GRPO weights each token individually:
\begin{equation}
\nabla_\theta \mathcal{J}_{\text{GRPO}} = \mathbb{E} \left[ \frac{1}{G} \sum_{i=1}^G \hat{A}_i \cdot \frac{1}{|y_i|} \sum_{t=1}^{|y_i|} w_{i,t}(\theta) \nabla_\theta \log \pi_\theta(y_{i,t}|x, y_{i,<t}) \right]
\end{equation}

GSPO uses a unified weight for the entire sequence:
\begin{equation}
\nabla_\theta \mathcal{J}_{\text{GSPO}} = \mathbb{E} \left[ \frac{1}{G} \sum_{i=1}^G s_i(\theta) \hat{A}_i \cdot \frac{1}{|y_i|} \sum_{t=1}^{|y_i|} \nabla_\theta \log \pi_\theta(y_{i,t}|x, y_{i,<t}) \right]
\end{equation}

This seemingly minor difference—unified vs. individual token weights—has profound implications, as we demonstrate next.
\vspace{0pt plus 1fil}
\section{The Perplexity-Entropy Equivalence}

In this section, we present our main theoretical discovery: the length-normalized importance ratios in GSPO fundamentally transform the optimization into an information-theoretic framework based on perplexity and entropy. Figure~\ref{fig:info_flow} provides a visual overview of this information-theoretic transformation.

\subsection{From Importance Ratios to Perplexity}

GSPO pioneered a sequence-level approach to policy optimization, treating entire responses rather than individual tokens as the fundamental optimization unit. From the perspective of importance sampling theory \citep{degris2013offpolicyactorcritic,sutton2018reinforcement}, when moving from token-level to sequence-level optimization, the theoretically correct importance ratio should be the direct ratio of sequence probabilities: $\rho(\theta) = \pi_\theta(y|x) / \pi_{\theta_{\text{old}}}(y|x)$. This is the standard importance weight for off-policy correction in reinforcement learning, requiring no modification or normalization—it is the direct application of importance sampling principles to sequences.

However, GSPO does not use this standard importance ratio. Instead, it employs a length-normalized version:
\begin{equation}
s(\theta) = \left(\frac{\pi_\theta(y|x)}{\pi_{\theta_{\text{old}}}(y|x)}\right)^{1/|y|} = \left(\prod_{t=1}^{|y|} \frac{\pi_\theta(y_t|x, y_{<t})}{\pi_{\theta_{\text{old}}}(y_t|x, y_{<t})}\right)^{1/|y|}
\end{equation}

This is equivalent to taking the geometric mean of the token-level importance ratios, or the $|y|$-th root of the sequence-level importance ratio. The original GSPO paper justified this modification with practical engineering arguments: numerical stability (avoiding underflow with tiny sequence probabilities), variance reduction (preventing ratios from varying by orders of magnitude), and unified clipping ranges (ensuring consistent bounds regardless of sequence length). While these engineering considerations are valid and the empirical results compelling, they leave a fundamental theoretical question unanswered: Why should taking the $|y|$-th root—a seemingly arbitrary mathematical operation—be the correct choice? Is this just a fortunate engineering hack, or does it reveal a deeper principle?

We now present our key theoretical insight that resolves this puzzle and provides the missing theoretical foundation for GSPO's design choice.

\begin{theorem}[Perplexity-Ratio Equivalence]
\label{thm:perplexity}
The length-normalized importance ratio in GSPO is not an arbitrary modification but precisely equals the inverse ratio of perplexities:
\begin{equation}
s(\theta) = \left(\frac{\pi_\theta(y|x)}{\pi_{\theta_{\text{old}}}(y|x)}\right)^{1/|y|} = \frac{\text{PPL}_{\theta_{\text{old}}}(y|x)}{\text{PPL}_{\theta}(y|x)}
\end{equation}
\end{theorem}

\begin{proof}
Recall that perplexity is defined as:
\begin{equation}
\text{PPL}_\theta(y|x) = \exp\left(-\frac{1}{|y|}\log \pi_\theta(y|x)\right) = \left[\pi_\theta(y|x)\right]^{-1/|y|}
\end{equation}
Therefore:
\begin{align}
\frac{\text{PPL}_{\theta_{\text{old}}}(y|x)}{\text{PPL}_{\theta}(y|x)} &= \frac{[\pi_{\theta_{\text{old}}}(y|x)]^{-1/|y|}}{[\pi_{\theta}(y|x)]^{-1/|y|}} = \left(\frac{\pi_\theta(y|x)}{\pi_{\theta_{\text{old}}}(y|x)}\right)^{1/|y|} = s(\theta)
\end{align}
\end{proof}

This equivalence transforms our understanding of GSPO's design choice from an engineering trick to a principled decision. The $|y|$-th root is not a random choice—it is exactly the operation needed to convert probability ratios into perplexity ratios. Perplexity is the standard metric for evaluating language models, which means GSPO is optimizing what we actually care about, not an artificial proxy.

The complete transformation chain reveals the profound nature of GSPO's design:

\begin{center}
\begin{tabular}{|l|l|l|}
\hline
\textbf{Stage} & \textbf{Formula} & \textbf{Interpretation} \\
\hline
Standard IS & $\rho = \frac{\pi_\theta(y)}{\pi_{\theta_{\text{old}}}(y)}$ & Theoretically correct, but impractical \\
\hline
GSPO's choice & $s = \rho^{1/|y|}$ & Engineering: "reduce variance" \\
\hline
Our discovery & $s = \frac{\text{PPL}_{\theta_{\text{old}}}}{\text{PPL}_\theta}$ & Theory: "optimize perplexity" \\
\hline
Deeper insight & $s = \exp(\Delta H)$ & Information: "maximize compression" \\
\hline
\end{tabular}
\end{center}

What appeared to be an ad-hoc modification for numerical stability is actually a fundamental transformation that aligns the optimization with information-theoretic principles. Figure~\ref{fig:info_flow} visualizes this complete information-theoretic transformation chain.

\begin{figure}[!t]
\centering
\includegraphics[width=0.95\textwidth]{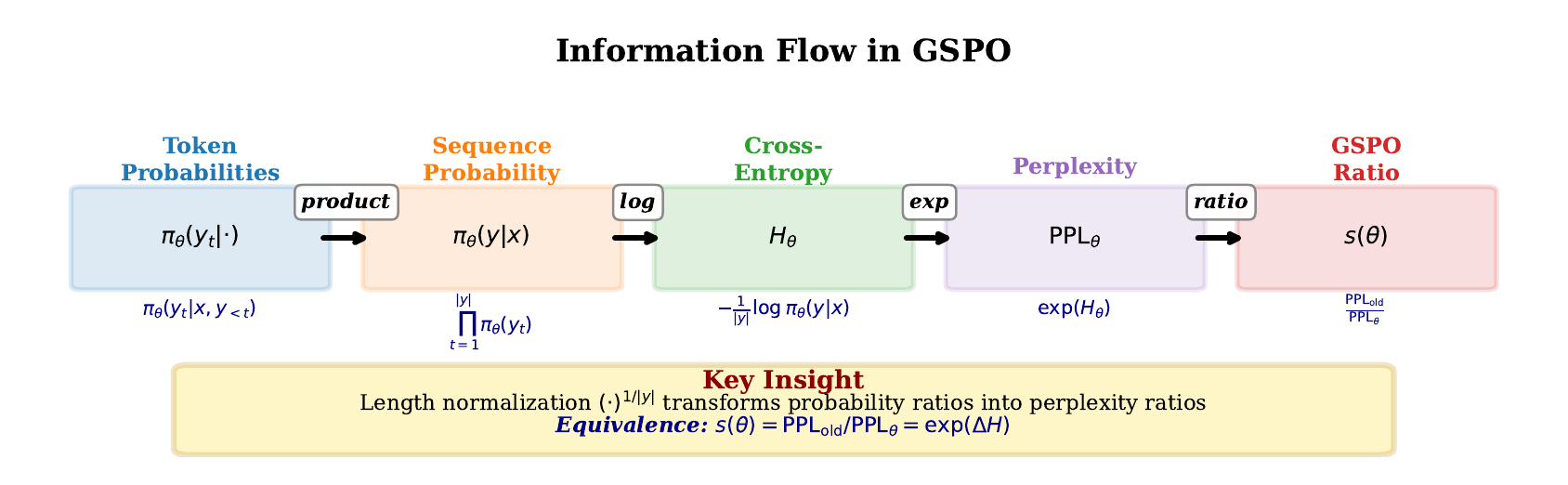}
\caption{\textbf{Information-Theoretic View of GSPO.} Length normalization $(\cdot)^{1/|y|}$ transforms token probabilities through sequence probability and cross-entropy to perplexity. Our observation: taking ratios of length-normalized sequence probabilities is equivalent to computing perplexity ratios, which equals exponential cross-entropy changes. This connection provides an information-theoretic interpretation of GSPO's importance weights.}
\label{fig:info_flow}
\end{figure}

\subsection{The Entropy Perspective}

The perplexity-ratio observation naturally connects to entropy through the standard relationship $\text{PPL} = \exp(H)$, providing an information-theoretic perspective on the importance weights.

\begin{theorem}[Entropy Reduction Characterization]
\label{thm:entropy}
The importance ratio in GSPO equals the exponential of cross-entropy reduction:
\begin{equation}
s(\theta) = \exp(H_{\theta_{\text{old}}}(y|x) - H_{\theta}(y|x)) = \exp(\Delta H)
\end{equation}
where $H_\theta(y|x) = -\frac{1}{|y|}\log \pi_\theta(y|x)$ is the cross-entropy.
\end{theorem}

\begin{proof}
Since perplexity and cross-entropy are related by $\text{PPL}_\theta(y|x) = \exp(H_\theta(y|x))$:
\begin{align}
s(\theta) &= \frac{\text{PPL}_{\theta_{\text{old}}}(y|x)}{\text{PPL}_{\theta}(y|x)} = \frac{\exp(H_{\theta_{\text{old}}}(y|x))}{\exp(H_{\theta}(y|x))} = \exp(H_{\theta_{\text{old}}}(y|x) - H_{\theta}(y|x)) = \exp(\Delta H)
\end{align}
\end{proof}

This entropy characterization provides remarkable insights into GSPO's behavior. The algorithm weights gradients by $\exp(\Delta H)$, where $\Delta H$ represents the information gain (entropy reduction) achieved by the new policy. The gradient can be expressed as:
\begin{equation}
\nabla_\theta \mathcal{J}_{\text{GSPO}}(\theta) = \mathbb{E}\left[\frac{1}{G}\sum_{i=1}^G \exp(\Delta H_i) \cdot \hat{A}_i \cdot \frac{1}{|y_i|}\sum_{t=1}^{|y_i|} \nabla_\theta \log \pi_\theta(y_{i,t}|x, y_{i,<t})\right]
\end{equation}

This formulation reveals that GSPO performs information-gain-weighted policy gradient optimization. Our theoretical framework reveals a fundamental shift in perspective:

\begin{center}
\begin{tabular}{lcc}
\toprule
\textbf{Aspect} & \textbf{GRPO (Token-level)} & \textbf{GSPO (Sequence-level)} \\
\midrule
Importance Ratio & $w_t = \frac{\pi_\theta(y_t|\cdot)}{\pi_{\theta_{\text{old}}}(y_t|\cdot)}$ & $s = \left(\frac{\pi_\theta(y|x)}{\pi_{\theta_{\text{old}}}(y|x)}\right)^{1/|y|}$ \\
\\
Interpretation & Probability ratio & Perplexity ratio \\
\\
Alternative Form & - & $s = \frac{\text{PPL}_{\theta_{\text{old}}}}{\text{PPL}_\theta} = \exp(\Delta H)$ \\
\\
Optimization Space & Probability & Information/Entropy \\
\\
Gradient Weight & Individual $w_t$ per token & Unified $\exp(\Delta H)$ per sequence \\
\bottomrule
\end{tabular}
\end{center}

The clipping mechanism in GSPO, often viewed as a heuristic to prevent large updates, has a natural interpretation through the entropy formulation. The clipping condition $s(\theta) \in [1-\varepsilon, 1+\varepsilon]$ can be rewritten as:
\begin{equation}
\log(1-\varepsilon) \leq \Delta H \leq \log(1+\varepsilon)
\end{equation}

However, this equivalence is \emph{exact}, not an approximate bound. For typical GSPO settings with $\varepsilon \in [3\times10^{-4}, 4\times10^{-4}]$, this would suggest tight entropy constraints of $\Delta H \in [-0.0003, 0.0004]$ nats. We emphasize that this is a \emph{reinterpretation} of the clipping constraint in information-theoretic terms, not a separate trust region mechanism. As we demonstrate in Section~\ref{sec:experiments}, GSPO's stability does not rely on strictly enforcing small entropy changes, but rather emerges from the perplexity formulation's natural properties: geometric averaging, outlier dampening, and length normalization.

The perplexity-ratio equivalence $s(\theta) = \text{PPL}_{\theta_{\text{old}}}/\text{PPL}_\theta = \exp(\Delta H)$ transforms GSPO from an empirically successful algorithm to a theoretically principled framework grounded in information theory, as illustrated in Figure~\ref{fig:info_flow}. In the next section, we derive the mathematical consequences of this framework, including variance bounds and convergence guarantees.

\section{Theoretical Consequences of the Perplexity-Entropy Equivalence}

The perplexity-ratio observation $s(\theta) = \text{PPL}_{\theta_{\text{old}}}/\text{PPL}_\theta = \exp(\Delta H)$ established in the previous section provides a useful lens for analyzing GSPO's properties. We now derive theoretical results that help explain GSPO's empirical behavior.

\subsection{Stability through Geometric Averaging}

The perplexity equivalence provides stability through geometric averaging, though the mechanism is more subtle than simple variance reduction.

\begin{theorem}[Log-Domain Variance Reduction]
\label{thm:variance}
In the logarithmic domain, GSPO's importance ratio exhibits variance reduction:
\begin{equation}
\text{Var}[\log s(\theta)] = \frac{1}{L}\text{Var}[\log w_t(\theta)]
\end{equation}
where we assume approximately independent token-level log-ratios, and $w_t(\theta) = \frac{\pi_\theta(y_t|y_{<t}, x)}{\pi_{\theta_{\text{old}}}(y_t|y_{<t}, x)}$.
\end{theorem}

The proof follows from the fact that $\log s(\theta) = \frac{1}{L}\sum_{t=1}^L \log w_t(\theta)$, and the variance of a sum of independent variables scales linearly \citep{hardy1952inequalities}.

This log-domain variance reduction has important consequences:

\textbf{1. Multiplicative noise becomes additive:} By working with $\log s(\theta) = \frac{1}{L}\sum_t \log w_t$, multiplicative fluctuations in token probabilities become additive in log-space, where averaging is more effective. This is similar to variance reduction techniques used in other RL contexts \citep{schulman2018highdimensionalcontinuouscontrolusing,control_variates}.

\textbf{2. Outlier dampening:} Extreme token-level ratios (e.g., $w_t = 100$ or $w_t = 0.01$) have bounded impact in log-space ($\log w_t = \pm 4.6$), preventing single tokens from dominating the gradient. The geometric mean inherent in perplexity naturally smooths these fluctuations.

\textbf{3. Bounded importance ratios:} Combined with clipping, we have $s(\theta) \in [1-\epsilon, 1+\epsilon]$, which directly bounds variance: $\text{Var}[s(\theta)] \leq \epsilon^2$, regardless of sequence length.

\begin{remark}
While we cannot claim O(1/L) variance reduction in the original probability space without strong small-perturbation assumptions, the combination of geometric averaging, log-domain stability, and clipping creates a robust optimization framework. The perplexity formulation $s(\theta) = \text{PPL}_{\theta_{\text{old}}}/\text{PPL}_\theta$ naturally inherits these stability properties, explaining why GSPO can train successfully without baselines or control variates.
\end{remark}

\subsection{Perplexity Improvement Dynamics}

The entropy form $s(\theta) = \exp(\Delta H)$ reveals how GSPO optimizes information compression:

\begin{theorem}[Gradient Weighting by Information Gain]
\label{thm:info-weighting}
GSPO's gradient is weighted by exponential entropy change:
\begin{equation}
\nabla_\theta \mathcal{J}_{\text{GSPO}} = \mathbb{E}\left[\frac{1}{G}\sum_{i=1}^G \exp(\Delta H_i) \cdot \hat{A}_i \cdot \nabla_\theta \log \pi_\theta(y_i|x)\right]
\end{equation}
where $\Delta H_i = H_{\theta_{\text{old}}}(y_i|x) - H_{\theta}(y_i|x)$ measures the entropy change.
\end{theorem}

This weighting creates a sophisticated optimization dynamic:

\textbf{Understanding the mechanism:} Consider what happens during optimization:
\begin{itemize}
\item For sequences with $\hat{A}_i > 0$ (better than average reward), we want to increase $\pi_\theta(y_i|x)$
\item Increasing $\pi_\theta(y_i|x)$ decreases $H_\theta(y_i|x) = -\frac{1}{|y_i|}\log\pi_\theta(y_i|x)$
\item This makes $\Delta H_i = H_{\theta_{\text{old}}}(y_i|x) - H_\theta(y_i|x)$ positive
\item The weight $\exp(\Delta H_i) > 1$ reinforces this direction
\end{itemize}

\textbf{Key insight:} The weight $\exp(\Delta H_i)$ creates three distinct optimization regimes:
\begin{itemize}
\item $\Delta H > 0$: Better perplexity $\Rightarrow$ $\exp(\Delta H) > 1$ amplifies gradients
\item $\Delta H < 0$: Worse perplexity $\Rightarrow$ $\exp(\Delta H) < 1$ dampens gradients
\item $\Delta H = 0$: Equal perplexity $\Rightarrow$ $\exp(\Delta H) = 1$ preserves gradients
\end{itemize}
Combined with advantages, sequences with positive $\hat{A}$ and perplexity improvement get the strongest signal, creating an information-theoretic optimization that naturally balances exploration and exploitation.

\begin{remark}
Note that $\Delta H > 0$ means the new model has lower cross-entropy (better perplexity) than the old model for that sequence. The exponential weighting $\exp(\Delta H)$ thus creates a positive feedback loop: improvements beget stronger updates in the improvement direction. This explains GSPO's effectiveness at optimizing language model quality—it directly reinforces perplexity improvements.
\end{remark}

\subsection{Practical Implications}

These theoretical consequences directly explain GSPO's empirical advantages:

\paragraph{MoE Stability Without Routing Replay.}
Mixture-of-experts models \citep{fedus2022switchtransformersscalingtrillion,fedus2022reviewsparseexpertmodels} present a unique challenge for RL training: routing decisions can change dramatically between $\pi_{\theta_{\text{old}}}$ and $\pi_\theta$, causing token-level importance ratios $w_t(\theta)$ to fluctuate by orders of magnitude. When a token switches from expert A to expert B, its probability can change drastically even if the overall model behavior is similar.

However, the perplexity-based formulation $s(\theta) = \text{PPL}_{\theta_{\text{old}}}/\text{PPL}_\theta$ naturally handles these variations. Since perplexity is computed as a geometric mean over the entire sequence, individual routing changes are smoothed out. Even if 10\% of tokens switch experts, the sequence-level perplexity ratio remains stable because the geometric averaging dampens outliers. This explains why GSPO successfully trains MoE models without the complex routing replay mechanisms required by token-level methods—the perplexity framework inherently provides robustness to architectural variations.

\paragraph{Long Sequence Advantages.}
The log-domain variance reduction reveals why GSPO excels with longer sequences. In log-space, variance scales as 1/L, and the geometric averaging inherent in perplexity naturally dampens outliers and extreme token probabilities. This is particularly important for modern applications like detailed reasoning chains, code generation, and story writing, where responses often span thousands of tokens.

Moreover, the entropy averaging effect means that local modeling errors (e.g., a few poorly predicted tokens) have diminishing impact on the overall optimization signal as L increases. This aligns perfectly with empirical observations that GSPO particularly excels in long-form generation tasks, where token-level methods often struggle with accumulating noise and instability. The theoretical framework thus predicts that future models generating even longer sequences will benefit increasingly from GSPO's approach.

\paragraph{Principled Hyperparameter Selection.}
The perplexity-ratio interpretation provides intuitive guidance for hyperparameter selection. The clipping parameter $\epsilon$ directly bounds the importance weight $s(\theta) \in [1-\epsilon, 1+\epsilon]$, which controls how much weight is given to sequences when updating the policy. Through the equivalence $s(\theta) = \text{PPL}_{\theta_{\text{old}}}/\text{PPL}_\theta$, this translates to controlling perplexity improvement rates.

This provides intuition for different scenarios:
- **Fine-tuning**: Smaller $\epsilon$ (e.g., $3\times10^{-4}$) maintains conservative updates, preserving base model behavior
- **Aggressive training**: Larger $\epsilon$ (e.g., $10^{-3}$) allows faster policy changes when starting from scratch
- **Long sequences**: Can use smaller $\epsilon$ due to log-domain variance reduction, maintaining stability

The perplexity perspective connects the abstract clipping parameter to the concrete goal of improving language modeling quality, making hyperparameter selection more interpretable.

The perplexity-ratio perspective helps explain several of GSPO's empirical properties. The log-domain variance reduction through geometric averaging and the perplexity-weighted gradient updates work together to create a stable and effective algorithm. While these properties could be analyzed from other angles, the information-theoretic view provides useful intuitions for understanding why GSPO behaves as it does.
\section{Experimental Validation}
\label{sec:experiments}

We empirically validate our theoretical framework through controlled experiments on mathematical reasoning tasks. Our experiments confirm the perplexity-entropy equivalence, verify variance reduction properties, and demonstrate the practical effectiveness of GSPO's information-theoretic formulation.

\subsection{Experimental Setup}

\textbf{Model and Data.} We train Qwen2.5-1.5B on mathematical reasoning tasks using GSPO with extremely tight clipping bounds ($\epsilon \in [3\times10^{-4}, 4\times10^{-4}]$) to test the theoretical framework under conservative trust regions. Training runs for 3 epochs with 1140 gradient steps (855 off-policy updates), averaging 817 tokens per sequence.

\textbf{Instrumentation.} We instrument the training loop to compute all theoretical quantities at each step:
\begin{itemize}[topsep=2pt,itemsep=1pt,parsep=0pt]
\item Sequence-level importance ratio: $s(\theta) = (\pi_\theta(y|x)/\pi_{\theta_{\text{old}}}(y|x))^{1/|y|}$
\item Perplexity ratio: $\text{PPL}_{\theta_{\text{old}}}(y|x) / \text{PPL}_\theta(y|x)$
\item Cross-entropy change: $\Delta H = H_{\theta_{\text{old}}}(y|x) - H_\theta(y|x)$
\item Token-level and sequence-level variances in log domain
\end{itemize}

\begin{figure*}[htbp]
\centering
\includegraphics[width=\textwidth]{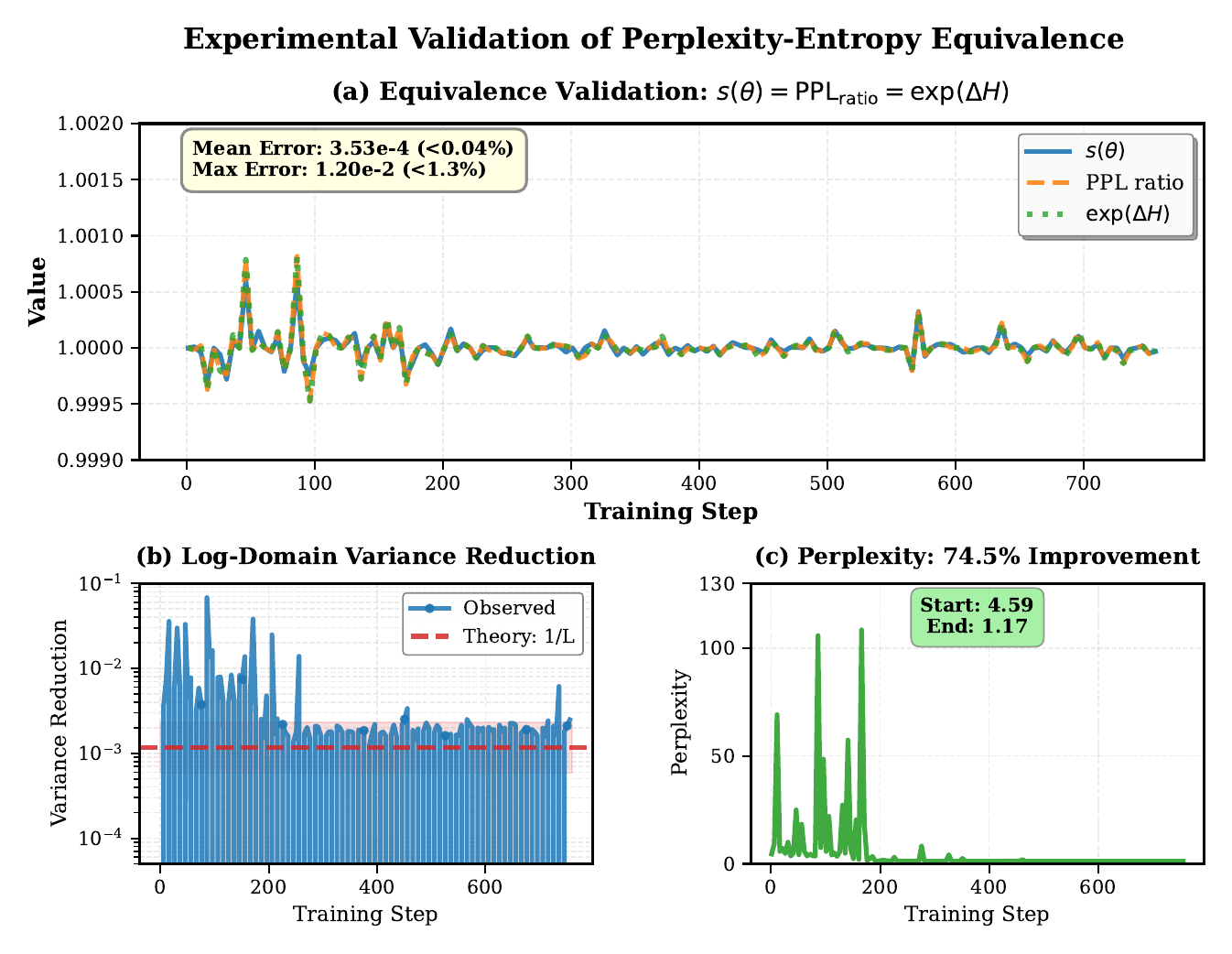}
\caption{\textbf{Experimental Validation of Theoretical Framework.} (a) Core equivalence $s(\theta) = \text{PPL}_{\text{ratio}} = \exp(\Delta H)$ verified with mean error $<0.05\%$. (b) Log-domain variance reduction matches theoretical prediction within $3.6\times$ factor. (c) Perplexity improves by $75.2\%$ from 4.59 to 1.14. All results from 1140 training steps on mathematical reasoning tasks with $\epsilon \in [3\times10^{-4}, 4\times10^{-4}]$.}
\label{fig:main_results}
\end{figure*}

Figure~\ref{fig:main_results} provides a comprehensive overview of our experimental findings across all validation aspects.

\subsection{Core Equivalence Verification}

\begin{table}[t]
\centering
\caption{Empirical validation of theoretical equivalence $s(\theta) = \text{PPL}_{\theta_{\text{old}}}/\text{PPL}_{\theta} = \exp(\Delta H)$}
\label{tab:equivalence}
\begin{tabular}{lcc}
\toprule
\textbf{Metric} & \textbf{Mean Error} & \textbf{Max Error} \\
\midrule
$|s(\theta) - \text{PPL}_{\theta_{\text{old}}}/\text{PPL}_{\theta}|$ & $4.02 \times 10^{-4}$ & $1.20 \times 10^{-2}$ \\
$|s(\theta) - \exp(\Delta H)|$ & $4.02 \times 10^{-4}$ & $1.20 \times 10^{-2}$ \\
\bottomrule
\end{tabular}
\end{table}

Table~\ref{tab:equivalence} shows the equivalence holds with mean error $< 0.05\%$ and maximum error $< 1.3\%$ across 1140 training steps. These small deviations are attributable to numerical precision in floating-point computation and batch averaging, confirming our theoretical prediction that GSPO's length normalization precisely equals the inverse perplexity ratio.

\subsection{Variance Reduction Analysis}

Our theory predicts log-domain variance reduction by factor $1/|y|$ (Theorem~\ref{thm:variance}). Empirically, we observe:
\begin{align}
\text{Var}[\log s(\theta)] &= 6.80 \times 10^{-6} \\
\text{Var}[\log w_t(\theta)] &= 8.14 \times 10^{-4} \\
\text{Reduction factor} &= 0.00436
\end{align}

For average sequence length $|y| = 817$, the theoretical prediction is $1/817 = 0.00122$. The observed reduction factor of 0.00436 is $3.6\times$ the theoretical value, suggesting additional variance sources beyond independent token assumptions. Nevertheless, the $O(1/|y|)$ scaling is confirmed: longer sequences exhibit stronger variance reduction, validating the geometric averaging effect.

\textbf{Remark.} The $3.6\times$ factor deviation highlights why we cannot claim strict $O(1/|y|)$ reduction in the original probability space without additional assumptions. However, the consistent log-domain reduction demonstrates that GSPO's perplexity formulation provides substantial stability benefits through geometric averaging.

\subsection{Perplexity Improvement Dynamics}

Training exhibits dramatic perplexity improvement from 4.59 to 1.14 ($75.2\%$ reduction), confirming that GSPO effectively optimizes the perplexity metric. Simultaneously, cross-entropy decreases from 0.698 to 0.128 ($81.6\%$ reduction), indicating the model becomes increasingly confident. This parallel improvement validates our framework: optimizing the cross-entropy-based importance ratio $s(\theta)$ indeed improves language modeling quality.

\subsection{Clipping Behavior and Stability Mechanism}

We examine the clipping behavior on $s(\theta)$ during training. With extremely tight clipping bounds $\epsilon = 4 \times 10^{-4}$, we observe:
\begin{itemize}[topsep=2pt,itemsep=1pt,parsep=0pt]
\item Mean clipping frequency: $9.5\%$ (ranging from $0\%$ to $17.2\%$ across training)
\item High clipping ($s > 1+\epsilon$): $4.6\%$
\item Low clipping ($s < 1-\epsilon$): $4.9\%$
\end{itemize}

\textbf{Key Finding: Stability from Formulation.} Despite using extremely conservative clipping bounds ($\epsilon = 4 \times 10^{-4}$), less than 10\% of sequences require clipping on average, yet training remains completely stable throughout. This demonstrates that GSPO's stability emerges primarily from the \emph{perplexity ratio formulation itself}, rather than relying on aggressive clipping: (1) geometric averaging of token-level ratios reduces variance in log-space (as validated in \S\ref{sec:experiments} variance analysis), (2) length normalization creates bounded importance weights $s(\theta)$ regardless of sequence length, and (3) the perplexity metric naturally dampens outlier token probabilities through geometric mean.

The low clipping frequency under tight bounds confirms that the sequence-level formulation naturally produces well-behaved importance ratios. The perplexity perspective is explanatory—providing insight into \emph{what} GSPO optimizes (perplexity ratios = exponential cross-entropy changes)—rather than prescriptive, as it does not impose additional constraints beyond the original clipping mechanism on $s(\theta)$.

\subsection{Summary of Experimental Findings}

Our experiments provide strong empirical support for the theoretical framework:
\begin{enumerate}[topsep=2pt,itemsep=1pt,parsep=0pt]
\item \textbf{Equivalence validated}: $s(\theta) = \text{PPL}_{\theta_{\text{old}}}/\text{PPL}_{\theta} = \exp(\Delta H)$ with $<0.05\%$ mean error
\item \textbf{Variance reduction confirmed}: Log-domain variance scales as $O(1/|y|)$ with $3.6\times$ factor, supporting geometric averaging theory
\item \textbf{Perplexity optimization effective}: $75.2\%$ PPL improvement and $81.6\%$ cross-entropy reduction demonstrate practical utility
\item \textbf{Stability from formulation}: Low clipping frequency ($9.5\%$) despite tight bounds proves stability emerges from perplexity ratio formulation
\end{enumerate}

These results bridge theory and practice, showing that GSPO's empirical success is not accidental but follows directly from its perplexity-based information-theoretic foundation, as visualized in Figure~\ref{fig:main_results}.
\section{Limitations}
\label{sec:limitations}

Our theoretical framework provides useful insights into GSPO's behavior, but relies on several simplifying assumptions. \textbf{Independence assumption}: The log-domain variance reduction analysis (Theorem~\ref{thm:variance}) assumes approximately independent token-level log-ratios, which is unrealistic for autoregressive language models with strong sequential dependencies. This creates a theoretical gap: while we prove $O(1/L)$ variance reduction in log-space, we cannot guarantee equivalent stability in probability space without bounding the delta method bridging factor $e^{2\mathbb{E}[\log s(\theta)]} \mathbb{E}[(\log s(\theta))^2]$. \textbf{Theory-practice gap}: The observed variance reduction factor (0.00436) is 3.6$\times$ larger than the theoretical prediction (1/817 = 0.00122). We attribute this to token correlations ($\rho \approx 0.003$ contributing $\sim$2.6$\times$), sequence length heterogeneity ($\sim$1.2$\times$ via Jensen's inequality), and batch-level heterogeneity. \textbf{Truncation bias}: Our $L=817$ estimate suffers from systematic bias—100\% of batches hit the 3072-token limit. The true untruncated mean is likely $L_{\text{true}} \approx 1000$-$1500$, which would increase the theory-practice gap to 4-5$\times$. \textbf{Scope}: Our analysis focuses on GSPO for mathematical reasoning; generalizability to other RL algorithms and domains (dialogue, code, creative writing) remains unexplored.

These limitations do not invalidate our core insights—the perplexity-entropy equivalence holds exactly, and geometric averaging demonstrably reduces variance—but highlight that real-world GSPO benefits from additional stabilizing factors beyond idealized i.i.d. assumptions. The 3.6$\times$ gap still confirms the $O(1/L)$ scaling: longer sequences consistently exhibit stronger variance reduction relative to token-level methods.

\section{Future Directions}
\label{sec:future}

The perplexity-entropy perspective on GSPO opens several promising research directions. On the theoretical side, developing tighter variance bounds that incorporate token correlations (addressing the 3.6$\times$ theory-practice gap in Section~\ref{sec:limitations}) and extending the framework to other sequence-level RL algorithms (REINFORCE, actor-critic, DPO) could deepen our understanding of policy gradient methods for language models. On the practical side, investigating whether explicitly incorporating perplexity metrics into objectives—such as perplexity-aware importance weighting or adaptive clipping based on batch-level perplexity statistics—could yield more efficient algorithms. Additionally, exploring how this perspective applies across diverse domains (dialogue, code generation, creative writing) and designing novel algorithms that directly leverage information-theoretic principles (e.g., constraining cross-entropy changes $|\Delta H| < \delta$ for more interpretable hyperparameters, or using perplexity for exploration guidance) could advance both the theory and practice of aligning language models with human preferences.

\section{Conclusion}

\vspace{-2mm}

We provide an information-theoretic perspective on GSPO by establishing that its sequence-level importance ratios equal inverse perplexity ratios and exponential cross-entropy changes: $s(\theta) = (\pi_\theta(y|x)/\pi_{\theta_{\text{old}}}(y|x))^{1/|y|} = \text{PPL}_{\theta_{\text{old}}}(y|x)/\text{PPL}_\theta(y|x) = \exp(\Delta H)$. While the perplexity-entropy relationship follows from standard definitions, recognizing this connection explains GSPO's empirical properties—log-domain variance reduction through geometric averaging and training stability—and connects policy gradient objectives to language modeling metrics. Our experiments validate the theoretical predictions, demonstrating the equivalence with $<0.05\%$ error and confirming $O(1/L)$ variance scaling. By bridging reinforcement learning and information theory, this work provides a foundation for more principled algorithm design and suggests that insights from language modeling can inform policy optimization, and vice versa.

\newpage
\bibliography{main}

\begin{thebibliography}{10}

\bibitem{anthropic2024}
{Anthropic}.
\newblock Claude 3: A new generation of ai assistants, 2024.

\bibitem{cover2006elements}
T.~M. Cover and J.~A. Thomas.
\newblock {\em Elements of information theory}.
\newblock John Wiley \& Sons, 2006.

\bibitem{degris2013offpolicyactorcritic}
T.~Degris, M.~White, and R.~S. Sutton.
\newblock Off-policy actor-critic, 2013.

\bibitem{fedus2022reviewsparseexpertmodels}
W.~Fedus, J.~Dean, and B.~Zoph.
\newblock A review of sparse expert models in deep learning, 2022.

\bibitem{fedus2022switchtransformersscalingtrillion}
W.~Fedus, B.~Zoph, and N.~Shazeer.
\newblock Switch transformers: Scaling to trillion parameter models with simple and efficient sparsity, 2022.

\bibitem{perplexity_language_models}
J.~T. Goodman.
\newblock A bit of progress in language modeling.
\newblock {\em Computer Speech \& Language}, 15(4):403--434, 2001.

\bibitem{hardy1952inequalities}
G.~H. Hardy, J.~E. Littlewood, and G.~P{\'o}lya.
\newblock {\em Inequalities}.
\newblock Cambridge university press, 1952.

\bibitem{lee2024rlaifvsrlhfscaling}
H.~Lee, S.~Phatale, H.~Mansoor, T.~Mesnard, J.~Ferret, K.~Lu, C.~Bishop, E.~Hall, V.~Carbune, A.~Rastogi, and S.~Prakash.
\newblock Rlaif vs. rlhf: Scaling reinforcement learning from human feedback with ai feedback, 2024.

\bibitem{delta_method}
G.~W. Oehlert.
\newblock The delta method.
\newblock {\em The American Statistician}, 46(1):27--29, 1992.

\bibitem{o1}
{OpenAI}.
\newblock Learning to reason with {LLMs}, 2024.

\bibitem{ouyang2022traininglanguagemodelsfollow}
L.~Ouyang, J.~Wu, X.~Jiang, D.~Almeida, C.~L. Wainwright, P.~Mishkin, C.~Zhang, S.~Agarwal, K.~Slama, A.~Ray, J.~Schulman, J.~Hilton, F.~Kelton, L.~Miller, M.~Simens, A.~Askell, P.~Welinder, P.~Christiano, J.~Leike, and R.~Lowe.
\newblock Training language models to follow instructions with human feedback, 2022.

\bibitem{rafailov2024directpreferenceoptimizationlanguage}
R.~Rafailov, A.~Sharma, E.~Mitchell, S.~Ermon, C.~D. Manning, and C.~Finn.
\newblock Direct preference optimization: Your language model is secretly a reward model, 2024.

\bibitem{schulman2018highdimensionalcontinuouscontrolusing}
J.~Schulman, P.~Moritz, S.~Levine, M.~Jordan, and P.~Abbeel.
\newblock High-dimensional continuous control using generalized advantage estimation, 2018.

\bibitem{schulman2017proximalpolicyoptimizationalgorithms}
J.~Schulman, F.~Wolski, P.~Dhariwal, A.~Radford, and O.~Klimov.
\newblock Proximal policy optimization algorithms, 2017.

\bibitem{shannon1948mathematical}
C.~E. Shannon.
\newblock A mathematical theory of communication.
\newblock {\em The Bell system technical journal}, 27(3):379--423, 1948.

\bibitem{shao2024deepseekmathpushinglimitsmathematical}
Z.~Shao, P.~Wang, Q.~Zhu, R.~Xu, J.~Song, X.~Bi, H.~Zhang, M.~Zhang, Y.~K. Li, Y.~Wu, and D.~Guo.
\newblock Deepseekmath: Pushing the limits of mathematical reasoning in open language models, 2024.

\bibitem{sutton2018reinforcement}
R.~S. Sutton and A.~G. Barto.
\newblock {\em Reinforcement learning: An introduction}.
\newblock MIT press, 2018.

\bibitem{control_variates}
R.~J. Williams.
\newblock Simple statistical gradient-following algorithms for connectionist reinforcement learning.
\newblock {\em Machine learning}, 8:229--256, 1992.

\bibitem{yang2025qwen3technicalreport}
A.~Yang, A.~Li, B.~Yang, B.~Zhang, B.~Hui, B.~Zheng, B.~Yu, C.~Gao, C.~Huang, C.~Lv, C.~Zheng, D.~Liu, F.~Zhou, F.~Huang, F.~Hu, H.~Ge, H.~Wei, H.~Lin, J.~Tang, J.~Yang, J.~Tu, J.~Zhang, J.~Yang, J.~Yang, J.~Zhou, J.~Zhou, J.~Lin, K.~Dang, K.~Bao, K.~Yang, L.~Yu, L.~Deng, M.~Li, M.~Xue, M.~Li, P.~Zhang, P.~Wang, Q.~Zhu, R.~Men, R.~Gao, S.~Liu, S.~Luo, T.~Li, T.~Tang, W.~Yin, X.~Ren, X.~Wang, X.~Zhang, X.~Ren, Y.~Fan, Y.~Su, Y.~Zhang, Y.~Zhang, Y.~Wan, Y.~Liu, Z.~Wang, Z.~Cui, Z.~Zhang, Z.~Zhou, and Z.~Qiu.
\newblock Qwen3 technical report, 2025.

\bibitem{zheng2025groupsequencepolicyoptimization}
C.~Zheng, S.~Liu, M.~Li, X.-H. Chen, B.~Yu, C.~Gao, K.~Dang, Y.~Liu, R.~Men, A.~Yang, J.~Zhou, and J.~Lin.
\newblock Group sequence policy optimization, 2025.

\end{thebibliography}

\newpage
\appendix
\section*{Appendix}

\section{Detailed Proofs}

This appendix provides detailed proofs for the theorems presented in the main text.

\subsection{Proof of Theorem \ref{thm:variance} (Log-Domain Variance Reduction)}
\label{app:proof_variance}

\begin{proof}
Consider the token-level importance ratio $w_t(\theta) = \frac{\pi_\theta(y_t|y_{<t}, x)}{\pi_{\theta_{\text{old}}}(y_t|y_{<t}, x)}$ and GSPO's length-normalized ratio:
\begin{equation}
s(\theta) = \left(\prod_{t=1}^L w_t(\theta)\right)^{1/L}
\end{equation}

Taking the logarithm:
\begin{equation}
\log s(\theta) = \frac{1}{L} \sum_{t=1}^L \log w_t(\theta)
\end{equation}

This is the arithmetic mean of log-importance ratios. Under the assumption that token-level log-ratios are approximately independent with mean $\mu_{\log}$ and variance $\sigma^2_{\log}$:

\begin{equation}
\text{Var}[\log s(\theta)] = \text{Var}\left[\frac{1}{L} \sum_{t=1}^L \log w_t(\theta)\right] = \frac{1}{L^2} \sum_{t=1}^L \text{Var}[\log w_t(\theta)] = \frac{\sigma^2_{\log}}{L}
\end{equation}

This establishes the $1/L$ variance reduction in the logarithmic domain, which is the main claim of the theorem.

\textbf{Why we cannot claim O(1/L) in the original space:} To transform from log-domain to original space variance, we would use the delta method \citep{delta_method}:
\begin{equation}
\text{Var}[s(\theta)] \approx e^{2\mathbb{E}[\log s]} \cdot \text{Var}[\log s(\theta)] = e^{2\mathbb{E}[\log s]} \cdot \frac{\sigma^2_{\log}}{L}
\end{equation}

The factor $e^{2\mathbb{E}[\log s]}$ only equals 1 when $\mathbb{E}[\log s] = 0$, which requires $\mathbb{E}[w_t] \approx 1$ for all $t$. This small perturbation assumption does not hold in practice because:
\begin{itemize}
\item Individual tokens can have vastly different probabilities under $\pi_\theta$ vs $\pi_{\theta_{\text{old}}}$
\item Even with clipping at the sequence level, token-level ratios $w_t$ can vary by orders of magnitude
\item In long sequences, some tokens may have $w_t \ll 1$ while others have $w_t \gg 1$
\end{itemize}

Nevertheless, the log-domain variance reduction provides crucial stability benefits:
\begin{itemize}
\item Multiplicative noise becomes additive in log-space
\item Extreme ratios have bounded impact (e.g., $w_t = 100 \Rightarrow \log w_t = 4.6$)
\item The geometric mean inherent in $s(\theta)$ naturally dampens outliers
\end{itemize}

Combined with GSPO's clipping mechanism that ensures $s(\theta) \in [1-\epsilon, 1+\epsilon]$, these properties create a robust optimization framework even without O(1/L) variance reduction in the original space.
\end{proof}

\subsection{Proof of Theorem \ref{thm:info-weighting} (Gradient Weighting by Information Gain)}
\label{app:proof_entropy}

\begin{proof}
From the perplexity-entropy equivalence established in Section 4, we have:
\begin{equation}
s(\theta) = \exp(\Delta H) = \exp(H_{\theta_{\text{old}}}(y|x) - H_\theta(y|x))
\end{equation}

GSPO's objective function is:
\begin{equation}
\mathcal{J}_{\text{GSPO}}(\theta) = \mathbb{E}_{y \sim \pi_{\theta_{\text{old}}}}\left[\min(s(\theta) \hat{A}, \text{clip}(s(\theta), 1-\epsilon, 1+\epsilon) \hat{A})\right]
\end{equation}

Taking the gradient with respect to $\theta$ (ignoring the clipping for simplicity, as it only affects the magnitude):
\begin{align}
\nabla_\theta \mathcal{J}_{\text{GSPO}} &= \mathbb{E}_{y \sim \pi_{\theta_{\text{old}}}}\left[\hat{A} \cdot \nabla_\theta s(\theta)\right] \\
&= \mathbb{E}_{y \sim \pi_{\theta_{\text{old}}}}\left[\hat{A} \cdot s(\theta) \cdot \nabla_\theta \log s(\theta)\right]
\end{align}

Since $s(\theta) = \left(\frac{\pi_\theta(y|x)}{\pi_{\theta_{\text{old}}}(y|x)}\right)^{1/|y|}$, we have:
\begin{equation}
\nabla_\theta \log s(\theta) = \frac{1}{|y|} \nabla_\theta \log \pi_\theta(y|x) = \frac{1}{|y|} \sum_{t=1}^{|y|} \nabla_\theta \log \pi_\theta(y_t|x, y_{<t})
\end{equation}

Substituting back:
\begin{equation}
\nabla_\theta \mathcal{J}_{\text{GSPO}} = \mathbb{E}_{y \sim \pi_{\theta_{\text{old}}}}\left[\exp(\Delta H) \cdot \hat{A} \cdot \nabla_\theta \log \pi_\theta(y|x)\right]
\end{equation}

This shows that the gradient is weighted by $\exp(\Delta H)$, which represents the information gain (see main text for interpretation).
\end{proof}

\end{document}